\documentclass[11pt,a4paper]{article}

\usepackage{amsmath}
\usepackage{amsthm}
\usepackage{amssymb}

\numberwithin{equation}{section}

\theoremstyle{definition}
\newtheorem{definition}{Definition}[section]

\theoremstyle{plain}

\newtheorem{proposition}[definition]{Proposition}
\theoremstyle{remark}

\usepackage{graphicx}
\usepackage{subcaption}
\usepackage{booktabs}
\usepackage{siunitx}
\usepackage{verbatim}

\usepackage[table]{xcolor}
\definecolor{bestbg}{RGB}{229,245,224}   
\definecolor{secondbg}{RGB}{222,235,247} 

\usepackage{enumitem}
\usepackage{pifont}

\usepackage[a4paper,
  left=20mm,right=20mm,
  top=18mm,bottom=22mm,
  includeheadfoot]{geometry}

\newenvironment{myitem}
  {\begin{itemize}[label=\textcolor{gray!70}{\ding{110}},
                   leftmargin=0em,
                   labelsep=0.5em,
                   itemindent=1.2em,
                   itemsep=0.5em,
                   topsep=0.5em,
                   parsep=0.5em,
                   partopsep=0pt]}
  {\end{itemize}}

\newenvironment{tritemize}
  {\begin{itemize}[label=\textcolor{gray!70}{\raisebox{0.2ex}{$\blacktriangleright$}}]}
  {\end{itemize}}

\usepackage[colorlinks=true, allcolors=blue]{hyperref}

\newcommand{\R}{\mathbb{R}}
\newcommand{\E}{\mathbb{E}}
\newcommand{\x}{\mathbf{x}}
\newcommand{\vv}{\mathbf{v}}
\newcommand{\y}{\mathbf{y}}
\newcommand{\z}{\mathbf{z}}
\newcommand{\n}{\boldsymbol{\epsilon}}

\newcommand{\A}{\mathcal{A}}
\newcommand{\F}{\mathcal{F}}
\newcommand{\G}{\mathcal{G}}
\newcommand{\D}{\mathcal{D}}
\newcommand{\N}{\mathcal{N}}

\newcommand{\B}{\mathcal{B}}
\newcommand{\U}{\mathcal{U}}
\newcommand{\PP}{\mathcal{P}}

\newcommand{\prox}{\operatorname{prox}}
\newcommand{\argmin}{\operatorname{argmin}}
\newcommand{\tv}{\operatorname{TV}}
\newcommand{\id}{\operatorname{Id}}

\newcommand{\Afull}{\mathbf{A}}
\newcommand{\M}{\mathbf{M}}
\newcommand{\eps}{\epsilon}

\title{HyDRA: Hybrid Denoising Regularization for Measurement-Only DEQ Training}

\author{
Markus Haltmeier\thanks{Department of Mathematics, University of Innsbruck, Innsbruck, Austria.}
\and
Lukas Neumann\thanks{Institute of Basic Sciences in Engineering Science, University of Innsbruck, Austria.}
\and
Nadja Gruber\thanks{Department of Computer Science, University of Innsbruck, Austria.}
\and
Johannes Schwab\thanks{University of Applied Sciences Kufstein, Austria.}
\and
Gyeongha Hwang\thanks{Department of Mathematics, Yeungnam University, Gyeongsan, Korea. Corresponding author: \href{mailto:ghhwang@yu.ac.kr}{ghhwang@yu.ac.kr}.}
}

\date{} 

\begin{document}
\maketitle

\begin{abstract}
Solving image reconstruction problems of the form \(\A\x = \y\) remains challenging due to ill-posedness and the lack of large-scale supervised datasets. Deep Equilibrium (DEQ) models have been used successfully but typically require supervised pairs \((\x,\y)\). In many practical settings, only measurements \(\y\) are available. We introduce HyDRA (Hybrid Denoising Regularization Adaptation), a measurement-only framework for DEQ training that combines measurement consistency with an adaptive denoising regularization term, together with a data-driven early stopping criterion. Experiments on sparse-view CT demonstrate competitive reconstruction quality and fast inference.
\end{abstract}

\section{Introduction}
 Inveres 
problems arise in a wide range of scientific and engineering applications, including geophysical exploration, remote sensing, nondestructive testing, medical imaging, and various forms of microscopy \cite{bertero1998introduction,engl1996regularization,hansen2010discrete,scherzer2009variational}. After discretization, the objective is to recover an unknown signal $\x \in \R^n$ from noisy and potentially incomplete measurements $\y \in \mathbb{R}^m$ governed by
\begin{equation} \label{eq:ip}
    \y = (\A \x ) \odot \n,
\end{equation}
where $\A$ denotes a linear forward operator and $(\cdot) \odot \n$ is the measurement noise model. Such problems are typically \textit{ill-posed} and  $\A$ may be ill-conditioned or possess a nontrivial null space, making inversion unstable and non-unique. We consider a probabilistic setting where images $\x$, $\y$ follow  distributions $p_\x$, $p_\y$, but $\A$ is a fixed deterministic and known map.

We have the particularly challenging scenario of a severely under-determined $\A$ in mind, which additionally has unstable pseudo-inverse. A prototypical example is limited-data CT, where $\A = \M \Afull $, with $\Afull$ representing the discretized full-data Radon transform and $\M$ a masking or subsampling operator indicating that only a subset of measurements is available. Reducing cost, acquisition time, and radiation exposure motivates avoiding full measurements in practice. Classical reconstruction techniques address these issues through explicit regularization that encodes prior knowledge about $\x$, such as smoothness or sparsity \cite{scherzer2009variational}. In the prototypical variational regularization
$\min_{\x}\; \|\A \x - \y\|_2^2 + \alpha\, r(\x)$, the spatial prior is included as an additive penalty $\alpha\, r(\x)$ with $\alpha > 0$, while the term $\|\A \x - \y\|_2^2$ enforces data consistency. This framework includes classical methods such as Tikhonov regularization and total variation (TV) minimization, which have demonstrated broad utility. However, these hand-crafted priors often fail to capture the statistical complexity of natural images and may introduce artifacts such as oversmoothing or staircasing.

In recent years, deep learning methods have achieved state-of-the-art performance across a variety of image reconstruction problems. Here, reconstruction is performed by a parameterized map $\B_\theta \colon Y \to X$, where the high-dimensional parameter $\theta$ is fitted to available data. Most existing  approaches rely on supervised learning. Typical architectures include explicit residual-type networks~\cite{jin2017deep}, null-space-type networks~\cite{schwab2019deep}, and unrolled networks~\cite{sun2016deep,rick2017one,adler2017solving}. Another class comprises fixed-point networks~\cite{bai2019deep,gilton2021deep}, where $\B_\theta(\y)$ is the fixed point of an implicit layer. Under the umbrella of Deep Equilibrium (DEQ) models~\cite{gilton2021deep}, these implicit networks employ implicit differentiation and Jacobian-free backpropagation~\cite{fung2022jfb}, enabling effectively infinite-depth representations with constant memory. Despite their strong performance, these supervised methods require ground-truth paired samples $(\x,\y)$ for training, which are often costly or impractical to obtain in medical and scientific imaging, thereby limiting their applicability.

Unsupervised deep learning methods have recently been developed that operate directly on measurement samples $\y$ without requiring ground-truth data. The absence of ground truth reconstructions gives rise to two major obstacles: first, the measurements contain noise, and the ill-conditioning of $\A$ causes noise amplification when minimizing the data-space loss $\|\A \B_\theta(\y) - \y\|^2$; second, difficulties arise when the operator $\A$ is highly underdetermined, as in limited-data Radon transforms. In this case, $\A$ has a large nullspace, and any data-space loss is unable to constrain the nullspace component of $\B_\theta$. Some approaches mitigate these issues by assuming $\A$ is random~\cite{millard2023theoretical,yaman2020self,bubba2025tomoselfdeq}, others assume an invertible forward map~\cite{hendriksen2020noise2inverse}, others add problem-structured regularizers \cite{chen2021equivariant}, and others employ surrogate measurement-space losses~\cite{gruber2024sparse2inverse}. Implicit fixed-point networks and DEQ methods, however, have not yet been used in the self-supervised context.

In this work, we propose a Hybrid Denoising Regularization Adaptation (HyDRA) framework for unsupervised image reconstruction using fixed-point networks. HyDRA is trained with an unsupervised loss that is the sum of a data-consistency term and a denoising-based regularization term that accounts for both the null space and the ill-posedness of the forward operator. HyDRA learns image structure directly from measurement samples $\y$ without requiring ground-truth samples $\x$. The framework includes a fixed-point architecture that allows incorporation of known priors, such as positivity or sparsity, and a data-driven early-stopping criterion that stabilizes training and prevents overfitting. Efficient inference is achieved through equilibrium solving, enabling fast reconstruction. We evaluate HyDRA on the challenging sparse-view CT reconstruction task and show that it recovers high-quality images from severely limited measurements, outperforming both classical variational methods and supervised deep learning baselines.

\section{Background}

Solving the inverse problem \eqref{eq:ip} amounts to finding a reconstruction map $\B \colon \mathbb{R}^m \to \mathbb{R}^n$ that maps noisy data $\y$ back to the unknown image $\x$. A wide range of methods have been proposed, assuming different levels of prior information to construct $\B$. In this work, we construct $\B$ by selecting a specific map from the parameterized family $(\B_\theta)_{\theta \in \Theta}$, where the parameter $\theta$ is determined during training. Different situations arise depending on the type of data available:

\begin{tritemize}
   \item \textsc{Single-shot:}
    Here the reconstruction map is optimized on the same measurement that it is later applied to, resulting in a single-instance or per-sample procedure, as used in DIPs or physics-based fitting.

    \item \textsc{Measurement-only:} 
    Here, only measurement samples \(\y\) are available. Learning relies on self-supervised or unsupervised approaches that do not require ground-truth \(\x\) samples.

 \item \textsc{Supervised learning:} Here matched pairs \((\x,\y)\) are available, enabling supervised training of mappings between the measurement domain and the target domain.
\end{tritemize}

The first two classes have the advantage of not requiring ground-truth data $\x$. The single-shot methods have the further advantage of not using training data at all, thus being cost-effective during data collection. However, training at inference yields time-consuming reconstruction, especially when combined with modern learning techniques. The unsupervised method allows  faster inference at the cost of the training phase  based on  $\y$-data.  Specific methods also differ in the architectures used. Below, we provide a concise overview, focusing on the concepts that underpin HyDRA.

\subsection{Single-shot reconstruction}

In this class of methods, no additional samples of $\x$, $\y$ besides the specific observed measurements are required. Classical regularization methods, DIP (DIP), as well as Plug-and-Play (PnP) fall into this category.

\begin{myitem}
\item \textsc{Regularization methods:} The first are classical regularization methods where $\B$ is selected from a family $(\B_\alpha)_{\alpha>0}$ that pointwise converges to a right inverse of $\A$ as $\alpha \to 0$. The methods we use here are FBP (filtered backprojection) or variational regularization, respectively:
\begin{align} \label{eq:reg-fbp}
\B_\alpha(\y) &\triangleq \A^T\big(\F_\alpha \y\big), 
\\ \label{eq:reg-var}
\B_\alpha(\y) &\triangleq \argmin_\x \,\|\A\x - \y\|^2 + \alpha\, r(\x) \, .
\end{align}
In FBP, regularization is achieved by the spectral filter $\F_\alpha$, whereas in variational  regularization  the prior information is incorporated via the regularization functional $r \colon \R^n \to [0,\infty]$. The latter is commonly approximated via iterative methods derived from the optimality condition to~\eqref{eq:reg-var}.  A prototypical example is the proximal gradient (PGD) method,
\begin{equation}\label{eq:PGD}
\x_{k+1} = \prox_{s\,\alpha\, r}\!\left(\x_k - 2 s\,\alpha\, \A^\top(\A\x_k - \y)\right),
\end{equation}
with step size $s>0$. Several learning-based reconstruction methods originate from PGD and its variants.

\item \textsc{PnP:} PnP methods integrate state-of-the-art denoisers with classical optimization. Building upon PGD, PnP replaces the proximal operator $\prox_{s \alpha r}$ with a (potentially learned) denoiser $\D_\alpha(\cdot)$, yielding the update
\begin{equation}\label{eq:pnp-PGD}
\x_{k+1} = \D_\alpha\!\big(\x_k - 2 s\, \A^\top(\A\x_k - \y)\big).
\end{equation}
When the denoiser is set to the proximal map ($\D_\alpha = \prox_{s \alpha r}$) of the scaled regularizer $s \alpha r$, this reduces to classical PGD~\eqref{eq:PGD}. The use of a general denoiser enables regularization through data-driven priors encoded in the trained denoiser $\D_\alpha$. Note that~\eqref{eq:pnp-PGD} is a fixed-point iteration for $\D_\alpha \circ \big(\id - 2 s\, \A^\top(\A(\cdot) - \y)\big)$. A regularization theory for such implicit networks has been derived in~\cite{ebner2024plug}. When $\D_\alpha$ is of non-gradient form, this extends classical variational regularization~\eqref{eq:reg-var}. Although PnP achieves state-of-the-art performance across diverse imaging modalities, its effectiveness heavily depends on the denoiser’s training domain, and performance can degrade when the denoiser is mismatched to the target data distribution.

\item \textsc{DIP:} A different line of single-shot methods is DIP, introduced in~\cite{ulyanov2018deep}. DIP exploits the implicit regularization of convolutional neural networks without requiring external training data. The reconstructed image is represented as $\B(\y, \z) = f_{\theta(\y, \z)}(\z)$, where $f_\theta$ is a network that uses the random vector $\z$ as input, and the parameters $\theta = \theta(\y, \z)$ are determined by minimizing
\begin{equation}\label{eq:dip}
    \min_\theta \, \big\| \A f_\theta(\z) - \y \big\|_2^2
\end{equation}
combined with early stopping. The implicit prior contained in the network architecture has demonstrated surprisingly high-quality reconstructions without any supervision or strong explicit regularizer. However, DIP suffers from long inference times and strong sensitivity to the stopping point, which heavily influences reconstruction quality.

\end{myitem}

\subsection{Supervised reconstruction}

Most existing state-of-the-art learning approaches rely on supervised learning.
These methods minimize the supervised loss $\E\!\left[\|\B_\theta(\y) - \x\|^2\right]$, where $\E$ denotes the expectation over the joint distribution of $(\x,\y)$.
In practice, given paired samples $(\x_i,\y_i)$, this loss is approximately minimized over a chosen architecture $(\B_\theta)_{\theta\in\Theta}$.
Proposed methods differ significantly in the employed architectures; we provide a brief summary below.

\begin{myitem}

\item \textsc{Residual networks:} Residual networks have the form $\B_\theta = (\id + \U_\theta)\, \A^\sharp$, where $\A^\sharp \colon \R^m \to \R^n$ is an initial reconstruction map (near-inverse) and $\U_\theta$ is a standard image-to-image network such as the U-net~\cite{ronneberger2015unet}.
While very simple, this approach turns out to be highly effective with high-quality results~\cite{jin2017deep}.
However, plain residual architectures can exhibit data mismatch for inputs outside the training distribution.
Null-space networks address this by enforcing data consistency $\A\x \simeq \y$ in the architecture.
Specifically, the null-space networks proposed in~\cite{schwab2019deep} are of the form $
\B_\theta = \big(\id + \PP_{\operatorname{ker}\A}\, \U_\theta\big)\, \A^\sharp_\alpha$
where $\A^\sharp_\alpha$ arises from a regularization method.
Opposed to standard residual networks, the  added component $\PP_{\operatorname{ker}\A}\, \U_\theta$ acts only in the null space of $\A$ and therefore does not alter data consistency.
Moreover,~\cite{schwab2019deep} shows that this yields a data-driven regularization method.
Variations that also allow changes in the orthogonal complement have been proposed in~\cite{chen2020deep,goppel2023data,schwab2020big}.

\item \textsc{Unrolled networks:} Another line of work aims to improve data consistency via network cascades that $N$ times alternate between an image-to-image network and a layer enforcing data consistency. A prototypical example is learned proximal gradient, where $\B_\theta(\y)$  is defined by finite number of updates $\x_{k+1} = \D_\theta^{(k)} (\x_k - 2s  \A^T(\A\x_k - \y) )$ and where $\D_\theta^{(k)}$ are standard image-to-image networks. Related approaches include learned gradient descent \cite{adler2017solving}, learned ADMM \cite{sun2016deep}, learned primal-dual~\cite{adler2018learned}, network cascades~\cite{schlemper2017deep,kofler2018u}, and variational networks~\cite{hammernik2018learning}. However, at least theoretically, such networks may still lack strict data consistency for measurements outside the training distribution.

\item \textsc{DEQ:}
While motivated by iterative schemes like~\eqref{eq:PGD}, unrolled networks are only loosely related, as they use different networks at each iteration and do not arise from an underlying fixed-point equation. This limitation is addressed by DEQ models, which are supervised implicit networks for solving~\eqref{eq:ip}. A representative example is DEQ-prox, where reconstructions are defined via the fixed point equation
\begin{equation}\label{eq:deq-PG}
    \x^*  \triangleq \D_\theta (\x^* - \A^T(\A\x^* - \y) ).
\end{equation}
When solved with a corresponding fixed-point iteration, this recovers \eqref{eq:pnp-PGD}. Unlike PnP, where $D_\theta$ is a fixed map such as  pretrained denoiser, DEQ trains the implicit network end-to-end.
DEQ enables training with constant memory and adaptive inference by solving for the fixed point using tools such as Anderson acceleration~\cite{walker2011anderson}. Gradients are obtained via implicit differentiation and can be approximated through Jacobian-free backpropagation~\cite{fung2022jfb}. 
\end{myitem}

\subsection{Unsupervised reconstruction (measurements-only)}

Unsupervised reconstruction determines the reconstruction map from samples of measurement data $\y$. The challenge is how to extract information about the images $\x$ to be recovered solely from $\y$.

\begin{myitem}

\item \textsc{Data-consistency loss:} One approach is to supervise a network $\B_\theta  = f_\theta \circ \A^\sharp$ via data-consistency (DC) on the training data, potentially adding a regularization term,
\begin{equation}\label{eq:unsuper}
    \min_{\theta}\; \E_\y \Bigl[ \bigl\| \A(\B_\theta \y) - \y \bigr\|^2 + \alpha\, r(\B_\theta \y) \Bigr] \,.
\end{equation}
Effectively, this solves the variational regularization jointly via the network $\B_\theta$ instead of minimizing~\eqref{eq:reg-var} separately for each $\x$. While quite natural and used in a few works~\cite{zhang2021general,hellwege2025unsupervised,tamir2019unsupervised,wang2020neural}, this approach seems not very common. One reason may be that it heavily depends on the choice of the regularizer $r$, which can introduce bias that interferes with the information contained in $\y$. On the other hand, in the unregularized case $\alpha = 0$, the method suffers from noise as well as the underdeterminedness of $\A$. A recent development \cite{chen2021equivariant} proposes the regularizer $r(\x)= \E_{g} \| (\B_\theta \A - \id) (T_g \x)\|^2$ designed to promote equivariance for a class of image-space transforms $(T_g)_{g \in G}$.

\item \textsc{Self-supervised loss:}
Self-supervision refers to generating secondary data from $\y$ (by adding noise or masking) that is then used to supervise the reconstruction. These approaches were first developed for image denoising~\cite{batson2019noise2self,moran2020noisier2noise}, where $\A$ is the identity, and were later extended to general inverse problems~\cite{millard2023theoretical,yaman2020self,bubba2025tomoselfdeq,hendriksen2020noise2inverse,gruber2024sparse2inverse,gruber2025noisier2inverse,schut2025equivariance2inverse,tachella2024unsure,chen2022robust}. 
A powerful class of approaches follows the self-supervised Noise2Self strategy~\cite{batson2019noise2self} and its extensions to inverse problems~\cite{hendriksen2020noise2inverse,gruber2024sparse2inverse}. Several works deal with random operators~\cite{bubba2025tomoselfdeq,millard2023theoretical,yaman2020self}, whereas we focus on a fixed forward operator $\A$ here. These issues are addressed in~\cite{gruber2024sparse2inverse}, where the proposed Sparse2Inverse (S2I) uses a reconstruction operator of the form $\B = f \circ \B_I$, with $\B_I$ denoting a linear reconstruction from a subset of measurements $\y_I$ and $f$ a neural network acting in the image domain. To account for null-space components due to incomplete measurements, the method employs a data-domain loss,
\begin{equation}\label{eq:s2i}
    \E\, \big\| A_I \! \big( f \circ \B^c_{I}(\y^c_{I}) \big) - \y_I \big\|^2.
\end{equation}
This formulation enables the network to learn null-space components and effectively suppress missing-data artifacts in sparse-angle CT, combining the benefits of self-supervision with the implicit prior induced by CNN architectures. However, relying on a data-space loss and comparing against noisy measurements introduces challenges related to regularization, stability, and the design of appropriate stopping criteria.

\end{myitem}

\section{HyDRA Framework}

In this section, we introduce HyDRA, a novel self-supervised framework for solving linear inverse problems. HyDRA uses an implicit fixed-point architecture and is trained with a self-supervised loss that uses samples of indirect and incomplete data $\y$ only. The proposed unsupervised loss combines an  data-consistency term with a self-supervised regularization term. An additional key component is a data-driven stopping criterion that automatically guides when to stop training to avoid overfitting. These elements are described in detail in the following subsections.

\subsection{HyDRA architecture}

The proposed HyDRA architecture employs a fixed-point formulation in which the learned reconstruction operator $\B_\theta$ is defined implicitly through a fixed-point equation.

\begin{definition}[HyDRA architecture] \label{def:HyDRA}
Suppose
\begin{tritemize}
\item $\A \colon \R^n \to \R^m$ is a given forward operator,
\item $(\D_\theta)_{\theta \in \Theta}$ is a family of contractions $\D_\theta \colon \R^n \to \R^n$,
\item $\Omega \colon \R^n \to \R^n$ is a nonexpansive map,
\item $s \in (0,1/\|\A\|^2)$ and $\lambda \in (0,1)$ are parameters.
\end{tritemize}
We define the HyDRA architecture $(\B_\theta)_{\theta \in \Theta}$ with $\B_\theta \colon \R^m \to \R^n$, $\y \mapsto \x^\ast$, implicitly through the fixed-point equation
\begin{equation} \label{eq:HyDRA}
    \x^\ast = \N_\theta(\G(\x^\ast,\y))\,,
\end{equation}
where
\begin{align}
    \G(\x,\y) &\triangleq \x -  2s\, \A^\top(\A\x - \y), \label{eq:HyDRA-G} \\
    \N_\theta(\vv) &\triangleq \Omega\!\bigl(\lambda\, \D_\theta(\vv) + (1-\lambda)\, \vv \bigr). \label{eq:HyDRA-N}
\end{align}
Here $\A$, $\Omega$, $s$, and $\lambda$ are considered fixed components of the architecture and are not included explicitly in the notation.
\end{definition}

According to Definition~\ref{def:HyDRA}, HyDRA realizes the reconstruction as the fixed point of \eqref{eq:HyDRA}. It incorporates a trainable regularization network $\D_\theta$ together with an operator $\Omega$ that can be used to enforce additional known physical constraints. The HyDRA fixed-point map in \eqref{eq:HyDRA} is the composition of the gradient map \eqref{eq:HyDRA-G} and the learnable component \eqref{eq:HyDRA-N}. Note that if $\Omega = \operatorname{Id}$ is the identity operator, then the HyDRA architecture reduces to the DE-Prox architecture proposed in~\cite{gilton2021deep}. In the case $\Omega = \mathcal{P}_+$, the projection onto the set of nonnegative images, HyDRA coincides with the architecture introduced in~\cite{bubba2025tomoselfdeq}. However, these methods are either fully supervised~\cite{gilton2021deep} or rely on random forward maps~\cite{bubba2025tomoselfdeq} where the data contains missing information. In contrast, we train \eqref{eq:HyDRA} in a fully self-supervised manner where none of the data samples contains information missing in $\A\x$.

A key question is the existence and uniqueness of fixed points; otherwise, the architecture \eqref{eq:HyDRA} would not be well defined.

\begin{proposition}[Existence and uniqueness]
For any $\y \in \R^m$ and $\theta \in \Theta$, the map $\N_\theta \circ \G(\cdot,\y) \colon \R^n \to \R^n$ is a contraction and thus has exactly one fixed point. In particular, $\B_\theta$ is a well-defined implicit network.
\end{proposition}

\begin{proof}
The map $\x \mapsto \A^\top(\A\x - \y)$ is affine with linear part $\A^\top \A$, whose spectrum lies in $[0,\|\A\|^2]$. Moreover, the eigenvalues of $\id - 2s  \A^\top\A$ lie in $[1 - 2s\|\A\|^2, 1]$, which shows that $\G(\cdot,\y)$ is nonexpansive for $0 \le s \le 1/\|\A\|^2$.
By assumption, $\D_\theta$ is a contraction    thus the convex combination $\lambda\, \D_\theta + (1-\lambda) \id $ is a contraction too. 
Since $\Omega$ is nonexpansive, $\N_\theta = \Omega \circ (\lambda\, \D_\theta + (1-\lambda) \id )$ is a contraction. Therefore, the composition $\N_\theta \circ \G(\cdot,\y)$ is a contraction and by Banach's fixed-point theorem, it has a unique fixed point.
\end{proof}

\subsection{HyDRA loss}

The parameters $\theta$ in HyDRA are determined by minimizing the hybrid loss
\begin{equation} \label{eq:HyDRA-loss}
    \mathcal{H}(\theta) =
     \E_{\y} \left[ 
      \bigl\| \A \B_\theta (\y) - \y \bigr\|_2^2
    \right]
       \\ +
    \gamma \E_{\y,\eps}  \Bigl[ 
     \bigl\| \D_\theta(\B_\theta (\y) + \eps)  - \B_\theta (\y) \bigr\|_2^2
    \Bigr],
\end{equation}
where $\eps$ denotes additive noise.  
It consists of an unsupervised data-consistency loss in measurement space, $\| \A \B_\theta (\y) - \y \|_2^2$, together with an additional adaptive denoising regularization term
\begin{equation}
\mathcal{L}_{\mathrm{R}}
   \triangleq 
   \E_{\y,\eps} \Bigl[ 
      \bigl\| \D_\theta(\B_\theta (\y) + \eps) - \B_\theta (\y) \bigr\|_2^2
    \Bigr].
\end{equation}
By minimizing the hybrid loss, the network $\D_\theta$ learns an implicit image prior that captures the underlying statistical structure of clean images and effectively suppresses noise and artifacts. Joint training of denoising and reconstruction provides a powerful regularization mechanism for $\B_\theta$, enhancing reconstruction fidelity and stability even in the absence of ground-truth supervision.

The key component of HyDRA is the self-supervised adaptive regularization term, which trains $\D_\theta$ to be an effective denoiser on the class of reconstructions produced by $\B_\theta$. In this way, the denoiser adapts to the reconstruction manifold, yielding a task-aligned implicit prior. This loss is explicitly designed to make the regularization network $\D_\theta$ function as an implicit denoiser.

\subsection{Automated Stopping}

As in other self-supervised methods using a data-space loss, proper early stopping is crucial for achieving high-quality reconstructions. This reflects the missing data and the inherent instability of inverting $\A$ in the presence of noisy measurements.

We propose a novel, data-driven early-stopping criterion by monitoring the Peak Signal-to-Noise Ratio (PSNR),
\begin{equation}\label{eq:validation}
    \mathrm{PSNR}\!\big(\B_\theta(\y),\, \B_\theta(\A\,\B_\theta(\y))\big).
\end{equation}
and halting training when this value reaches its maximum. We refer HyDRA with automated stopping based on \eqref{eq:validation} as HyDRA-auto, opposed to any manual stopping.    

The rationale is as follows. A well-trained reconstruction operator $\B_\theta$ should be robust and stable. Taking its output $\B_\theta(\y)$, which is assumed to be a clean estimate of the signal, and generating synthetic measurements $\A \B_\theta(\y)$, the operator $\B_\theta$ should be able to reconstruct these synthetic measurements accurately. If the early-stopping metric begins to decrease, this indicates a divergence between $\B_\theta(\y)$ and $\B_\theta(\A \B_\theta(\y))$, signaling that $\B_\theta$ has begun to overfit to noise and artifact patterns present in the training measurements $\y$.

\section{Numerical experiments}
\label{sec:results}

To validate the effectiveness of HyDRA, we consider the sparse-view CT reconstruction problem. The forward model for sparse-view CT is formulated in the form \eqref{eq:ip} where $\A = \M \Afull$, with $\Afull \in \mathbb{R}^{p \times n}$ denoting the full-angle Radon transform and $\M \in \{0,1\}^{m \times p}$ a sampling matrix that selects a sparse subset of angular measurements. 
All experiments are performed on a workstation equipped with an Intel Core i9--10900K CPU and an NVIDIA RTX~3090 GPU.

\subsection{Data set}

For training and evaluation we use the low-dose parallel-beam (LoDoPaB) CT dataset~\cite{leuschner2021lodopab}, which contains two-dimensional CT images: $35{,}820$ training, $3{,}522$ validation, and $3{,}553$ test samples. Each image has a spatial resolution of $362 \times 362$ pixels. For our experiments, we retain the original data split but randomly select $10\%$ of each split to construct smaller training, validation, and test subsets. Sparse-view CT sinograms are synthesized using a parallel-beam acquisition geometry with $16$, $32$, and $64$ projection angles and $362$ equidistant detector elements per projection. To simulate low-intensity photon counts, Poisson noise corresponding to $1000$ photons per detector pixel is applied to the projection data.

\subsection{Compared methods}

We compare HyDRA against methods designed for a fixed forward operator $\A$ that are either single-shot estimators or are trained using only measurement samples $\y$ (self-supervised). This excludes (i) supervised methods, (ii) methods relying on randomized $\A$ during training, and (iii) approaches that explicitly encode invariances. We also include ablations that remove or replace key components of HyDRA. 

The compared self-supervised methods are:
\begin{tritemize}
\item FBP: Filtered backprojection \eqref{eq:reg-fbp}.
\item TV-reg: Variational regularization \eqref{eq:reg-var} with $r=\tv$.
\item DIP: single-shot Deep Image Prior via \eqref{eq:dip}.
\item S2I: U-Net trained with the self-supervised loss \eqref{eq:s2i}.
\item DEQ-plain: \eqref{eq:unsuper} with a DEQ model and $r=0$.
\item DEQ-TV: \eqref{eq:unsuper} with a DEQ model and $r=\tv$.
\item HyDRA-max: Proposed method with optimal stopping.
\item HyDRA-auto: Proposed method with  automated stopping.
\end{tritemize}

Here, FBP, TV-reg, and DIP are single-shot methods. S2I is a state-of-the-art self-supervised baseline that does not incorporate explicit invariance priors. DEQ-plain and DEQ-TV share the same measurement-space loss as HyDRA but omit HyDRA’s additional components, serving as ablations to quantify their impact.

\subsection{Implementation Details}

For the fixed-point architecture (Definition~\ref{def:HyDRA}), the network $\D_\theta$ is a U-Net, $\Omega$ denotes the projection onto the cube $[0,1]^n$, and parameters are chosen as $\gamma = \|\A\|^2$ and $\lambda = 0.4$ (selected empirically to balance stability and convergence speed of the equilibrium iterations). To solve the DEQ fixed-point equation \eqref{eq:HyDRA}, Anderson acceleration~\cite{walker2011anderson, scieur2016regularized} is used; the iteration proceeds until the relative norm of the difference between successive iterates falls below $10^{-3}$ or a maximum of $50$ iterations is reached. For minimizing the loss \eqref{eq:HyDRA-loss} with the HyDRA architecture, we use Jacobian-Free Backpropagation~\cite{fung2022jfb}. The noise level for the denoising regularization term $\mathcal{L}_{\mathrm{R}}$ is fixed at $15\%$, empirically determined to balance reconstruction fidelity and noise robustness. Optimization is performed using Adam~\cite{kingma2015adam} with a learning rate of $10^{-5}$ and a batch size of $1$. Training is conducted for $5 \times 10^{4}$, $1.5 \times 10^{5}$, and $2 \times 10^{6}$ iterations for the $16$-, $32$-, and $64$-view sparse reconstruction settings, respectively. The validation metric for early stopping is evaluated every $10^{3}$ iterations.

For TV-reg and DEQ-TV, the regularization parameter  is selected by a small grid search $\alpha \in \{10^{-4}, 10^{-3}, 10^{-2}, 10^{-1}\}$ to maximize PSNR with respect to the ground-truth images. For DIP, we adopt the implementation of \ \cite{ulyanov2018deep} and use two standard stabilizers: (i) additive input noise and (ii) an exponential moving average of recent network outputs \cite{cheng2019bayesian}. DIP is trained with early stopping at the PSNR peak, capped at $5000$ iterations. S2I \cite{gruber2024sparse2inverse} is included as a self-supervised baseline tailored to sparse-view CT and is trained for $500$ epochs with early stopping at the best PSNR. Both DEQ and DEQ-TV use the same early-stopping rule  based on \eqref{eq:validation}.

As can be seen, TV regularization provides a strong traditional baseline and offers substantial improvements over FBP, particularly in terms of SSIM. DIP achieves high PSNR values but exhibits reduced structural fidelity in the $16$- and $64$-view cases, where its SSIM falls below that of TV regularization. S2I performs competitively across all configurations. Notably, it achieves the highest SSIM among the baselines in the $32$- and $64$-view settings, highlighting its ability to preserve structural consistency without relying on ground-truth images. Its PSNR, however, remains slightly lower than that of DIP for the $16$-view setup. HyDRA is terminated automatically by our proposed self-supervised stopping criterion.

\begin{table}[htb!]
\caption{PSNR, SSIM, and inference time across methods for 16/32/64 projections. Best per column is highlighted in green; second-best in blue. FBP reference uses 1000 angles and roughly 4000 photons and is not part of the  comparison.}
\label{tab:results}
\centering
\resizebox{\textwidth}{!}{
\begin{tabular}{
    l
    S[table-format=2.2] S[table-format=1.3] S[table-format=3.4]
    S[table-format=2.2] S[table-format=1.3] S[table-format=3.4]
    S[table-format=2.2] S[table-format=1.3] S[table-format=3.4]
}
\toprule
\textbf{Method}
& \multicolumn{3}{c}{\textbf{16}} 
& \multicolumn{3}{c}{\textbf{32}} 
& \multicolumn{3}{c}{\textbf{64}} \\
\cmidrule(lr){2-4} \cmidrule(lr){5-7} \cmidrule(lr){8-10}
& \textbf{PSNR} & \textbf{SSIM} & \textbf{Time (s)}
& \textbf{PSNR} & \textbf{SSIM} & \textbf{Time (s)}
& \textbf{PSNR} & \textbf{SSIM} & \textbf{Time (s)} \\
\midrule
\multicolumn{10}{l}{\textit{Full Views, High Dose}}\\
FBP reference & 28.93 & 0.661 & \multicolumn{1}{c}{--}
                            & 28.93 & 0.661 & \multicolumn{1}{c}{--}
                            & 28.93 & 0.661 & \multicolumn{1}{c}{--} 
                            \\ \midrule
\addlinespace
\multicolumn{10}{l}{\textit{Classical baselines}}\\
FBP (Sparse)   & 7.55  & 0.221 & {\cellcolor{bestbg}\bfseries 0.0037}
               & 9.94  & 0.235 & {\cellcolor{bestbg}\bfseries 0.0044}
               & 12.52 & 0.258 & {\cellcolor{bestbg}\bfseries 0.0048} \\
TV-reg         & {\cellcolor{secondbg}26.24} & {\cellcolor{bestbg}\bfseries 0.702} & 4.998
               & 27.06 & 0.664 & 5.0972
               & 28.41 & 0.723 & 5.2463 \\
\addlinespace
\multicolumn{10}{l}{\textit{Learning baselines}}\\
DIP            & {\cellcolor{bestbg}\bfseries 26.25} & 0.606   & 207.79
               & {\cellcolor{bestbg}\bfseries 27.58} & 0.650   & 213.99
               & 28.75                               & 0.691   & 214.19 \\
S2I            & 23.91 & 0.582 & {\cellcolor{secondbg}0.0851}
               & 26.79 & 0.674 & {\cellcolor{secondbg}0.0355}
               & {\cellcolor{secondbg}29.31} & {\cellcolor{bestbg}\bfseries 0.746} & {\cellcolor{secondbg}0.0630} \\
\addlinespace
\multicolumn{10}{l}{\textit{Ablations}}\\
DEQ-plain      & 25.06 & 0.569 & 0.3408
               & 26.29 & 0.590 & 0.4916
               & 26.98 & 0.637 & 0.3485 \\
DEQ-TV         & 25.12 & 0.572 & 0.3702
               & 26.27 & 0.610 & 0.4795
               & 27.17 & 0.650 & 0.3131 \\
\addlinespace
\multicolumn{10}{l}{\textit{Proposed}}\\
HyDRA-max          & 25.54 & {\cellcolor{secondbg}0.628} & 0.4448
               & {\cellcolor{secondbg}27.28} & {\cellcolor{bestbg}\bfseries 0.680} & 0.4676
               & {\cellcolor{bestbg}\bfseries 29.50} & {\cellcolor{secondbg}0.733} & 0.5017 \\
HyDRA-auto     & 25.50 & 0.621 & 0.4587
               & 27.08 & {\cellcolor{secondbg}0.678} & 0.4329
               & 29.24 & 0.726 & 0.3795 \\
\bottomrule
\end{tabular}
}
\end{table}

\subsection{Results}

Table~\ref{tab:results} summarizes reconstruction performance in terms of PSNR and Structural Similarity Index Measure (SSIM) across the sparse-view settings ($16$, $32$, and $64$ projections).
HyDRA-max, demonstrates highly competitive performance with a superior balance of metrics. This is also supported by   visual results shown in Figures~\ref{fig:examples_16},~\ref{fig:examples_32} and~\ref{fig:examples_64}. HyDRA proves to be more robust than S2I; it consistently outperforms S2I in PSNR across all view configurations. Furthermore, the performance of HyDRA-auto remains comparable to HyDRA-max, validating that our proposed stopping criterion effectively identifies a near-optimal operating point without ground-truth access. Note also that 
in the $64$-view setting, the HyDRA-auto variant achieves the highest PSNR across all methods, surpassing both S2I and DIP.  HyDRA-auto for $64$ angles even approaches the baseline result obtained using $1000$ angles and a higher photon count of approximately $4000$ photons per bin. The ablation studies (DEQ and DEQ-TV, which use the same data-consistency as HyDRA but with either no regularization or a classical regularization term) demonstrate that the denoising regularization term $\mathcal{L}_{\mathrm{R}}$ is an essential ingredient of HyDRA, yielding significant improvements in reconstruction performance.

\begin{figure}[htb!]
\centering
\caption{Example reconstructions for 16 views.}
\label{fig:examples_16}
\includegraphics[width=\textwidth]{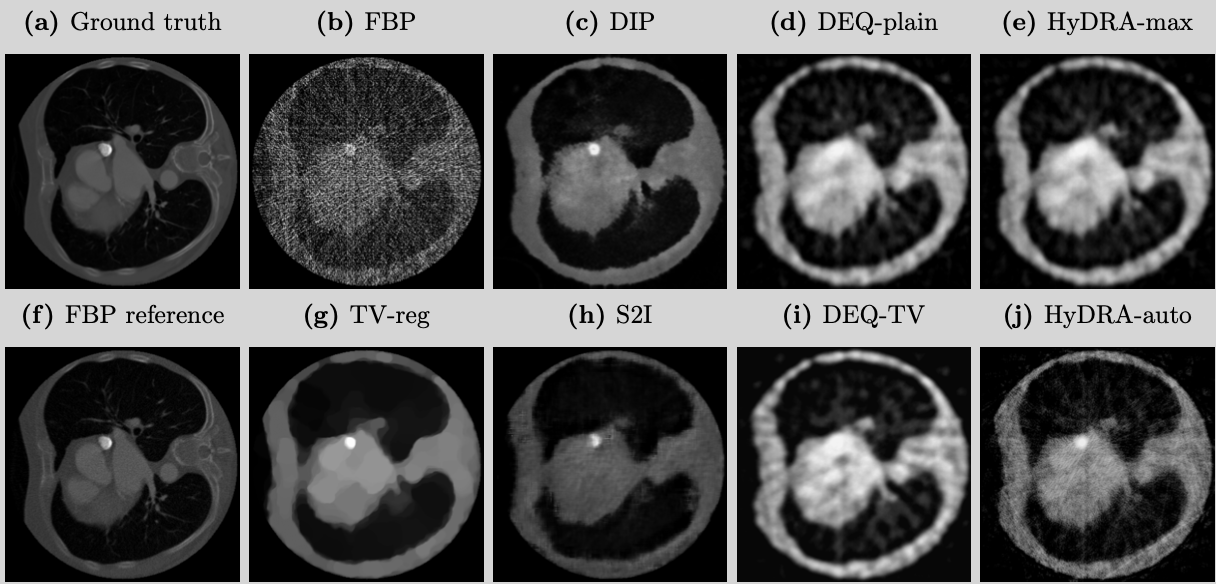}
\end{figure}

\begin{figure}[htb!]
\centering
\caption{Example reconstructions for 32 views.}
\label{fig:examples_32}
\includegraphics[width=\textwidth]{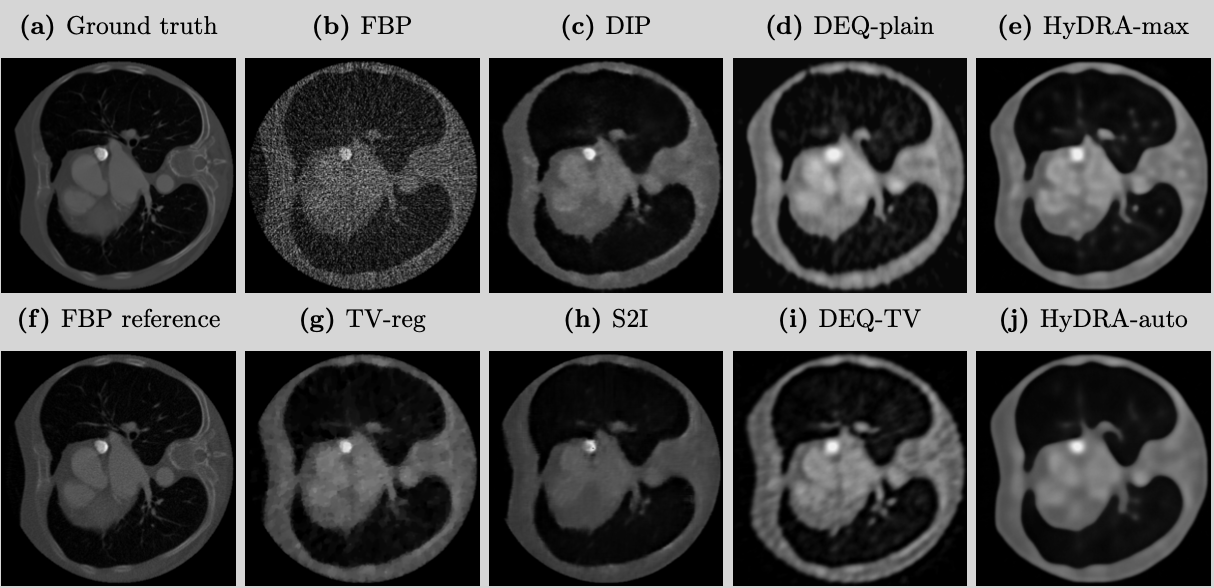}
\end{figure}

\begin{figure}[htb!]
\centering
\caption{Example reconstructions for 64 views.}
\label{fig:examples_64}
\includegraphics[width=\textwidth]{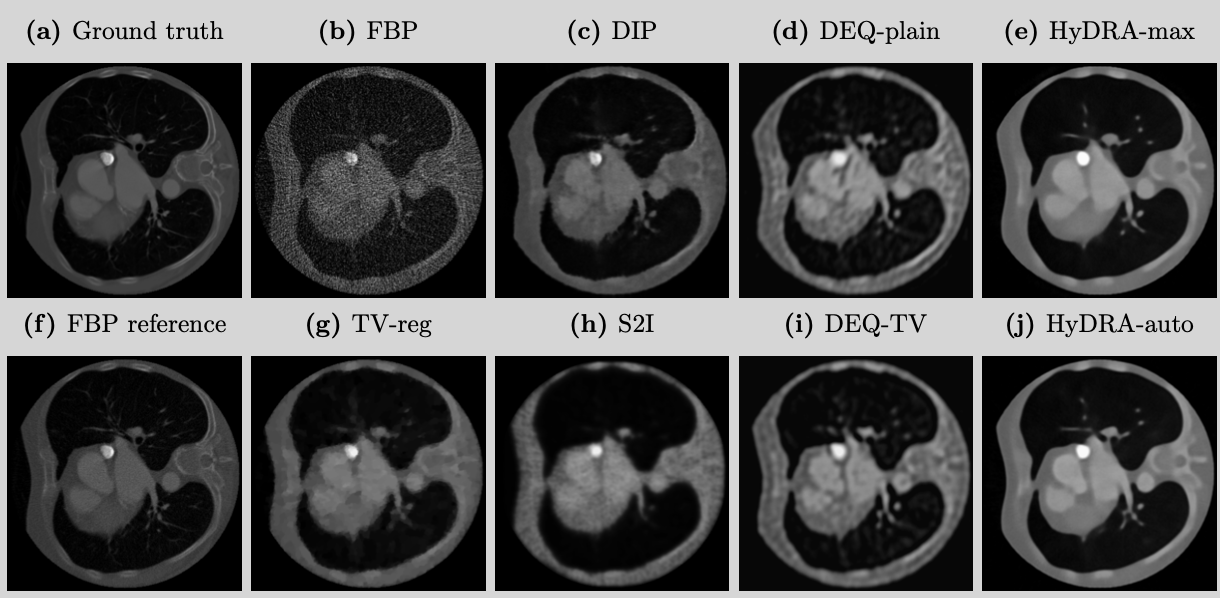}
\end{figure}

The computational efficiency results in Table \ref{tab:results} reveal notable differences across methods. DIP requires more than 200 seconds per reconstruction, making it impractical for real-time or large-scale applications. S2I is significantly faster, requiring only tens of milliseconds per reconstruction, reflecting its simple feed-forward architecture. HyDRA achieves a favorable balance, requiring roughly $0.4$--$0.5$ seconds per reconstruction—an order of magnitude faster than TV-reg and dramatically more efficient than DIP, while maintaining high reconstruction quality.

\section{Conclusion}
In this paper, we introduced HyDRA (Hybrid Denoising Regularization Adaptation), a novel self-supervised framework for solving ill-posed inverse problems based on DEQ models. HyDRA successfully integrates the memory-efficient, fixed-point inference of DEQ with a powerful, internally trained denoising regularizer. This hybrid design is trained end-to-end using a self-supervised loss that learns an implicit image prior directly from the measurements, completely obviating the need for paired ground-truth data. The proposed early-stopping criterion further enhances robustness by preventing overfitting to noise.

Our validation on the sparse-view CT reconstruction task showed HyDRA's performance. It achieves high-quality reconstructions, outperforming all baselines in the 64-view setting in both PSNR and SSIM. The most significant practical advantage of HyDRA is its computational efficiency. By leveraging Anderson acceleration for equilibrium solving, HyDRA achieves inference speeds approximately $470\times$ faster than DIP and $10\times$ faster than classical TV-regularization.

This work bridges the gap between the implicit regularization of self-supervised methods and the structured efficiency of equilibrium models. By providing high reconstruction quality, ground-truth-free training, and practically fast inference, HyDRA presents a compelling and viable solution for real-world inverse problems. Future research could involve applying HyDRA to a wider range of imaging modalities, such as MRI and computational microscopy.

\section*{Acknowledgment}
This work was supported by the National Research Foundation of Korea(NRF) grant funded by the Korea government(MSIT) (RS-2024-00333393).


\begin{thebibliography}{10}

\bibitem{adler2017solving}
Jonas Adler and Ozan {\"O}ktem.
\newblock Solving ill-posed inverse problems using iterative deep neural
  networks.
\newblock {\em Inverse Problems}, 33(12):124007, 2017.

\bibitem{adler2018learned}
Jonas Adler and Ozan \"Oktem.
\newblock Learned primal-dual reconstruction.
\newblock {\em IEEE Transactions on Medical Imaging}, 37(6):1322--1332, 2018.

\bibitem{bai2019deep}
Shaojie Bai, J.~Zico Kolter, and Vladlen Koltun.
\newblock Deep equilibrium models.
\newblock In {\em Advances in Neural Information Processing Systems}, pages
  690--701, 2019.

\bibitem{batson2019noise2self}
Joshua Batson and Loic Royer.
\newblock Noise2self: Blind denoising by self-supervision.
\newblock In {\em International conference on machine learning}, pages
  524--533. PMLR, 2019.

\bibitem{bertero1998introduction}
Mario Bertero and Patrizia Boccacci.
\newblock {\em Introduction to Inverse Problems in Imaging}.
\newblock CRC Press, 1998.

\bibitem{bubba2025tomoselfdeq}
Tatiana~A Bubba, Matteo Santacesaria, and Andrea Sebastiani.
\newblock Tomoselfdeq: Self-supervised deep equilibrium learning for
  sparse-angle ct reconstruction.
\newblock In {\em International Conference on Scale Space and Variational
  Methods in Computer Vision}, pages 334--346. Springer, 2025.

\bibitem{chen2020deep}
Dongdong Chen and Mike~E Davies.
\newblock Deep decomposition learning for inverse imaging problems.
\newblock In {\em European Conference on Computer Vision}, pages 510--526.
  Springer, 2020.

\bibitem{chen2021equivariant}
Dongdong Chen, Juli{\'a}n Tachella, and Mike~E Davies.
\newblock Equivariant imaging: Learning beyond the range space.
\newblock In {\em Proceedings of the IEEE/CVF International Conference on
  Computer Vision}, pages 4379--4388, 2021.

\bibitem{chen2022robust}
Dongdong Chen, Juli{\'a}n Tachella, and Mike~E Davies.
\newblock Robust equivariant imaging: a fully unsupervised framework for
  learning to image from noisy and partial measurements.
\newblock In {\em Proceedings of the IEEE/CVF Conference on Computer Vision and
  Pattern Recognition}, pages 5647--5656, 2022.

\bibitem{cheng2019bayesian}
Zezhou Cheng, Matheus Gadelha, Subhransu Maji, and Daniel Sheldon.
\newblock A bayesian perspective on the deep image prior.
\newblock In {\em Proceedings of the IEEE/CVF Conference on Computer Vision and
  Pattern Recognition}, pages 5443--5451, 2019.

\bibitem{ebner2024plug}
Andrea Ebner and Markus Haltmeier.
\newblock Plug-and-play image reconstruction is a convergent regularization
  method.
\newblock {\em IEEE Transactions on Image Processing}, 33:1476--1486, 2024.

\bibitem{engl1996regularization}
Heinz~Werner Engl, Martin Hanke, and Andreas Neubauer.
\newblock {\em Regularization of Inverse Problems}.
\newblock Springer, 1996.

\bibitem{fung2022jfb}
Samy~Wu Fung, Howard Heaton, Qiuwei Li, Daniel McKenzie, Stanley Osher, and
  Wotao Yin.
\newblock Jfb: Jacobian-free backpropagation for implicit networks.
\newblock In {\em Proceedings of the AAAI Conference on Artificial
  Intelligence}, volume~36, pages 6648--6656, 2022.

\bibitem{gilton2021deep}
Davis Gilton, Greg Ongie, and Rebecca Willett.
\newblock Deep equilibrium models for inverse problems in imaging.
\newblock {\em IEEE Transactions on Computational Imaging}, 7:1123--1133, 2021.

\bibitem{goppel2023data}
Simon G{\"o}ppel, J{\"u}rgen Frikel, and Markus Haltmeier.
\newblock Data-proximal null-space networks for inverse problems.
\newblock {\em arXiv:2309.06573}, 2023.

\bibitem{gruber2024sparse2inverse}
Nadja Gruber, Johannes Schwab, Elke Gizewski, and Markus Haltmeier.
\newblock Sparse2inverse: self-supervised inversion of sparse-view ct data.
\newblock {\em arXiv:2402.16921}, 2024.

\bibitem{gruber2025noisier2inverse}
Nadja Gruber, Johannes Schwab, Markus Haltmeier, Ander Biguri, Clemens Dlaska,
  and Gyeongha Hwang.
\newblock Noisier2inverse: Self-supervised learning for image reconstruction
  with correlated noise.
\newblock {\em IEEE Access}, 13:139445--139459, 2025.

\bibitem{hammernik2018learning}
Kerstin Hammernik, Teresa Klatzer, Erich Kobler, Michael~P Recht, Daniel~K
  Sodickson, Thomas Pock, and Florian Knoll.
\newblock Learning a variational network for reconstruction of accelerated mri
  data.
\newblock {\em Magnetic Resonance in Medicine}, 79(6):3055--3071, 2018.

\bibitem{hansen2010discrete}
Per~Christian Hansen.
\newblock {\em Discrete Inverse Problems: Insight and Algorithms}.
\newblock SIAM, 2010.

\bibitem{hellwege2025unsupervised}
Laura Hellwege, Johann~Christopher Engster, Moritz Schaar, Thorsten~M Buzug,
  and Maik Stille.
\newblock Unsupervised learning for inverse problems in computed tomography.
\newblock {\em arXiv preprint arXiv:2508.05321}, 2025.

\bibitem{hendriksen2020noise2inverse}
Allard~Adriaan Hendriksen, Dani{\"e}l~Maria Pelt, and K~Joost Batenburg.
\newblock Noise2inverse: Self-supervised deep convolutional denoising for
  tomography.
\newblock {\em IEEE Transactions on Computational Imaging}, 6:1320--1335, 2020.

\bibitem{jin2017deep}
Kyong~Hwan Jin, Michael~T McCann, Emmanuel Froustey, and Michael Unser.
\newblock Deep convolutional neural network for inverse problems in imaging.
\newblock {\em IEEE Transactions on Image Processing}, 26(9):4509--4522, 2017.

\bibitem{kingma2015adam}
Diederik~P. Kingma and Jimmy Ba.
\newblock Adam: A method for stochastic optimization.
\newblock {\em International Conference on Learning Representations (ICLR)},
  2015.

\bibitem{kofler2018u}
Andreas Kofler, Markus Haltmeier, Christoph Kolbitsch, Marc Kachelrie{\ss}, and
  Marc Dewey.
\newblock A u-nets cascade for sparse view computed tomography.
\newblock In {\em International Workshop on Machine Learning for Medical Image
  Reconstruction}, pages 91--99. Springer, 2018.

\bibitem{leuschner2021lodopab}
Johannes Leuschner, Maximilian Schmidt, Daniel~Otero Baguer, and Peter Maass.
\newblock Lodopab-ct, a benchmark dataset for low-dose computed tomography
  reconstruction.
\newblock {\em Scientific Data}, 8(1):109, 2021.

\bibitem{millard2023theoretical}
Charles Millard and Mark Chiew.
\newblock A theoretical framework for self-supervised mr image reconstruction
  using sub-sampling via variable density noisier2noise.
\newblock {\em IEEE transactions on computational imaging}, 9:707--720, 2023.

\bibitem{moran2020noisier2noise}
Nick Moran, Dan Schmidt, Yu~Zhong, and Patrick Coady.
\newblock Noisier2noise: Learning to denoise from unpaired noisy data.
\newblock In {\em Proceedings of the IEEE/CVF conference on computer vision and
  pattern recognition}, pages 12064--12072, 2020.

\bibitem{rick2017one}
JH~Rick~Chang, Chun-Liang Li, Barnabas Poczos, BVK Vijaya~Kumar, and Aswin~C
  Sankaranarayanan.
\newblock One network to solve them all--solving linear inverse problems using
  deep projection models.
\newblock In {\em Proceedings of the IEEE International Conference on Computer
  Vision}, pages 5888--5897, 2017.

\bibitem{ronneberger2015unet}
Olaf Ronneberger, Philipp Fischer, and Thomas Brox.
\newblock U-net: Convolutional networks for biomedical image segmentation.
\newblock In {\em Medical Image Computing and Computer-Assisted Intervention
  (MICCAI)}, pages 234--241. Springer, 2015.

\bibitem{scherzer2009variational}
Otmar Scherzer, Markus Grasmair, Harald Grossauer, Markus Haltmeier, and Frank
  Lenzen.
\newblock {\em Variational methods in imaging}, volume 167.
\newblock Springer, 2009.

\bibitem{schlemper2017deep}
Jo~Schlemper, Jose Caballero, Joseph~V Hajnal, Anthony~N Price, and Daniel
  Rueckert.
\newblock A deep cascade of convolutional neural networks for dynamic mr image
  reconstruction.
\newblock {\em IEEE transactions on Medical Imaging}, 37(2):491--503, 2017.

\bibitem{schut2025equivariance2inverse}
Dirk~Elias Schut, Adriaan Graas, Robert van Liere, and Tristan van Leeuwen.
\newblock Equivariance2inverse: A practical self-supervised ct reconstruction
  method benchmarked on real, limited-angle, and blurred data.
\newblock {\em arXiv:2510.23317}, 2025.

\bibitem{schwab2019deep}
Johannes Schwab, Stephan Antholzer, and Markus Haltmeier.
\newblock Deep null space learning for inverse problems: convergence analysis
  and rates.
\newblock {\em Inverse Problems}, 35(2):025008, 2019.

\bibitem{schwab2020big}
Johannes Schwab, Stephan Antholzer, and Markus Haltmeier.
\newblock Big in japan: Regularizing networks for solving inverse problems.
\newblock {\em Journal of mathematical imaging and vision}, 62(3):445--455,
  2020.

\bibitem{scieur2016regularized}
Damien Scieur, Alexandre d'Aspremont, and Francis Bach.
\newblock Regularized nonlinear acceleration.
\newblock In {\em Advances in Neural Information Processing Systems}, pages
  712--720, 2016.

\bibitem{sun2016deep}
Jian Sun, Huibin Li, Zongben Xu, et~al.
\newblock Deep admm-net for compressive sensing mri.
\newblock {\em Advances in neural information processing systems}, 29, 2016.

\bibitem{tachella2024unsure}
Juli{\'a}n Tachella, Mike Davies, and Laurent Jacques.
\newblock Unsure: self-supervised learning with unknown noise level and stein's
  unbiased risk estimate.
\newblock {\em arXiv:2409.01985}, 2024.

\bibitem{tamir2019unsupervised}
Jonathan~I Tamir, Stella~X Yu, and Michael Lustig.
\newblock Unsupervised deep basis pursuit: Learning inverse problems without
  ground-truth data.
\newblock {\em arXiv preprint arXiv:1910.13110}, 2019.

\bibitem{ulyanov2018deep}
Dmitry Ulyanov, Andrea Vedaldi, and Victor Lempitsky.
\newblock Deep image prior.
\newblock In {\em Proceedings of the IEEE Conference on Computer Vision and
  Pattern Recognition (CVPR)}, pages 9446--9454, 2018.

\bibitem{walker2011anderson}
Homer~F. Walker and Peng Ni.
\newblock Anderson acceleration for fixed-point iterations.
\newblock {\em SIAM Journal on Numerical Analysis}, 49(4):1715--1735, 2011.

\bibitem{wang2020neural}
Alan~Q Wang, Adrian~V Dalca, and Mert~R Sabuncu.
\newblock Neural network-based reconstruction in compressed sensing mri without
  fully-sampled training data.
\newblock In {\em international workshop on machine learning for medical image
  reconstruction}, pages 27--37. Springer, 2020.

\bibitem{yaman2020self}
Burhaneddin Yaman, Seyed Amir~Hossein Hosseini, Steen Moeller, Jutta Ellermann,
  K{\^a}mil U{\u{g}}urbil, and Mehmet Ak{\c{c}}akaya.
\newblock Self-supervised learning of physics-guided reconstruction neural
  networks without fully sampled reference data.
\newblock {\em Magnetic resonance in medicine}, 84(6):3172--3191, 2020.

\bibitem{zhang2021general}
Jingke Zhang, Qiong He, Congzhi Wang, Hongen Liao, and Jianwen Luo.
\newblock A general framework for inverse problem solving using self-supervised
  deep learning: validations in ultrasound and photoacoustic image
  reconstruction.
\newblock In {\em 2021 IEEE International Ultrasonics Symposium (IUS)}, pages
  1--4. IEEE, 2021.

\end{thebibliography}

\end{document}